\documentclass[letterpaper,11pt,reqno]{amsart}

\makeatletter
\usepackage{amssymb}
\usepackage{latexsym}
\usepackage{amsbsy}
\usepackage{amsfonts}
\usepackage{hyperref}
\usepackage{graphicx}
\usepackage{enumerate}
\usepackage{enumitem}
\usepackage{color}

\usepackage{tikz}

\def\marginpar#1{\ignorespaces}

\newtheorem{theorem}{Theorem}[section]
\newtheorem{lemma}[theorem]{Lemma}
\newtheorem{proposition}[theorem]{Proposition}

\newtheorem{Assumption}[theorem]{Assumption}

\numberwithin{equation}{section}
\makeatother
\begin{document}
\title[Learning MNLs]{Learning an arbitrary mixture of two multinomial logits}

\author[Wenpin Tang]{{Wenpin} Tang}
\address{Department of Industrial Engineering and Operations Research, Columbia University
} \email{wt2319@columbia.edu}

\date{\today} 
\begin{abstract}
In this paper, we consider mixtures of multinomial logistic models (MNL), which are known to $\epsilon$-approximate any random utility model. 
Despite its long history and broad use, rigorous results are only available for learning a uniform mixture of two MNLs.
Continuing this line of research, we study the problem of learning an arbitrary mixture of two MNLs.
We show that the identifiability of the mixture models may only fail on an algebraic variety of a negligible measure.
This is done by reducing the problem of learning a mixture of two MNLs to the problem of solving a system of univariate quartic equations.
We also devise an algorithm to learn any mixture of two MNLs using a polynomial number of samples and a linear number of queries, provided that a mixture of two MNLs over some finite universe is identifiable. 
Several numerical experiments and conjectures are also presented.
\end{abstract}

\maketitle

%\setcounter{tocdepth}{1}
%\tableofcontents
%-------------------------------------------------------------------------------------------------
\section{Introduction}

\quad This paper is concerned with the problem of learning a mixture of two multinomial logistic models from data.
Understanding an individual, or a user's rational behavior when facing a list of alternatives is a classical topic in economic theory.
In the era of data deluge, it has wide applications, especially for recommender systems where a user decides which of several competing products to purchase, 
and companies like Amazon, Netflix and Yelp look for which products are most relevant to a specific user.
A powerful tool to study user behavior is {\em discrete choice models}, 
and we refer to McFadden \cite{MF74, MF81, MF83, MF01} for the modern literature.
The most well-studied class of discrete choice models are the class of {\em random utility models}, 
which find roots in the work of Thurstone \cite{Thur27}, and were formally introduced by Marschak \cite{Mar60}.
The book of Train \cite{Train09} contains a thorough review on this subject. 
{\em Mixtures of multinomial logistic models} (also known as {\em mixed logits}) are a family of parametric random utility models, which have been widely used since 1980 \cite{BM80, CD80}, following earlier works of Bradley-Terry \cite{BT52},
and Luce-Plackett \cite{Luce, Plackett} on the multinomial logits.
Despite its broad use in practice, e.g. the EM algorithm \cite{DLR77}, there are few works on efficient algorithms to learn any non-trivial mixture of multinomial logistic models.
Chierichetti et al. \cite{CKT18} took the first step to develop polynomial-time algorithms to learn a uniform mixture of two multinomial logits.
They also pointed out that generalizing the results to non-uniform mixtures, or mixtures of more than two components is challenging.

\quad The purpose of this paper is to go beyond the uniform mixture, and study the problem of reconstruction and polynomial-time algorithms to learn an arbitrary mixture of two multinomial logits from user data.
To proceed further, we give a little more background.
A {\em multinomial logistic model}, or simply an {\em MNL} over a universe $\mathcal{U}$ is specified by 
a mapping from any non-empty subset $S \subseteq \mathcal{U}$ to a distribution over $S$.
The set $S$ is referred to as the {\em choice set} or the {\em slate}, from which a user selects exactly one item.
The MNL requires a {\em weight function} $w: \mathcal{U} \to \mathbb{R}_{+}$ which gives a positive weight to each item in the universe $\mathcal{U}$. 
The model then assigns probability to each $u \in S$ proportional to its weight:
\begin{equation*}
\mathbb{P}(u | S): = \frac{w(u)}{\sum_{v \in S} w(v)}, \quad \mbox{for each } u \in S.
\end{equation*}
One can regard $\mathbb{P}(u|S)$ as the conditional probability of selecting item $u$ given the alternatives in $S$.
We normalize the weight function $w$ by $\sum_{u \in \mathcal{U}} w(u) = 1$, so
$w : \mathcal{U} \to \Delta_{|\mathcal{U}|-1}$ 
where $\Delta_{|\mathcal{U}|-1}:=\{(a_1, \ldots, a_{|\mathcal{U}|}) \in \mathbb{R}^{|\mathcal{U}|}_{+}: \sum_{i = 1}^{|\mathcal{U}|} a_i = 1\}$ is the $(|\mathcal{U}| - 1)$-simplex
with $|\mathcal{U}|$ the number of items in $\mathcal{U}$.
Given sufficient data of slates $S$ with resulting choices of $u \in S$, 
it is possible to estimate the weight $w$ via maximum likelihood estimation.
The underlying problem is convex, and is easy to solve by gradient methods.

\quad In spite of simple interpretation and computational advantages, MNL is criticized for being too restrictive on the model behavior across related subsets, and thus lack of flexibility. 
This drawback is due to the fact that MNL is defined as a family of functions mapping any $S \subseteq \mathcal{U}$ to a distribution over $S$, based on a single fixed weight function.
One way to resolve this issue is to remove the constraint that the likelihood of each item is always proportional to a fixed weight. 
The aforementioned random utility model does the job: it is defined by a distribution over vectors, where each vector assigns a value to each item of $\mathcal{U}$.
A user then draws a random vector from this distribution, and selects the item of $S$ with the largest value.
McFadden and Train \cite{MT00} observed that any random utility model can be approximated arbitrarily close by a mixture of MNLs. 
Thus, learning general random utility models reduces to learning mixtures of MNLs.
This is the reason why mixtures of MNLs are widely recognized by practitioners.
However, almost all existing learning approaches are empirical, 
and there are few provable results on learning non-trivial mixtures of MNLs.
The only exception is \cite{CKT18}, 
where the authors resolved positively the problem of leaning a uniform mixture of MNLs,
assuming an oracle access to the true distributions over all choice sets.
Indeed, they do not learn the mixture of MNLs from sample data. 
Here we take a further step to study a possibly non-uniform mixture of two MNLs given access to samples.
We provide a few positive results on optimal learning algorithms.

\quad The object of interest is specified by the triple $(a,b,\mu)$,
where $a, b: \mathcal{U} \to \Delta_{|\mathcal{U}|-1}$ are two weight functions, and $\mu \in (0,1)$ is the {\em mixing weight}.
A $\mu$-mixture of two MNLs $(a,b)$ assigns to item $u$ in the set $S \subseteq \mathcal{U}$ the probability
\begin{equation*}
\mathbb{P}^{\mu}(u | S) := \mu \, \frac{a(u)}{\sum_{v \in S} a(v)} + (1 - \mu) \, \frac{b(u)}{\sum_{v \in S} b(v)}.
\end{equation*}
The goal of the learning problem is to reconstruct the parameters $(a,b, \mu)$ 
from the sample outcomes of the slate induced by the mixture.
In this paper, we assume that the mixing weight $\mu \in (0,1)$ is known.
So the problem consists of learning the weight functions $(a,b)$ in a $\mu$-mixture of MNLs.
Our program is illustrated by the following diagram, which consists of two basic problems:
\begin{figure}[h]
\centering
\includegraphics[width=0.5 \columnwidth]{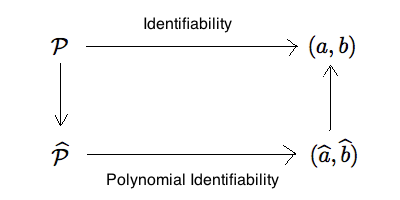}
\caption{$\mathcal{P}$ is the true distribution of the $\mu$-mixture of MNLs, and $\widehat{\mathcal{P}}$ is an estimate of $\mathcal{P}$.}
\end{figure}

\begin{enumerate}[itemsep = 3 pt]
\item[(i)]
{\em Identifiability}: 
Is the $\mu$-mixture model $\mathcal{P}$ uniquely characterized by a pair of weights $(a,b)$?
That is, if the $\mu$-mixtures of MNLs $(a,b)$ and $(a',b')$ agree on each $S \subseteq \mathcal{U}$ then $a = a'$, $b = b'$.
\item[(ii)]
{\em Polynomial identifiability}: Given that the $\mu$-mixture model $\mathcal{P}$ is identifiable, is there an algorithm to learn or to estimate the weights $(\widehat{a}, \widehat{b})$ with polynomial running time and using a polynomial number of samples?
\end{enumerate}

\quad We start with the identifiability of the $\mu$-mixture of two MNLs.
The main result of \cite{CKT18} showed that for $\mu = 1/2$, 
$(i)$ if $|\mathcal{U}| \ge 3$, any uniform mixture of two MNLs is identifiable;
$(ii)$ there is an algorithm which learns any uniform mixture of two MNLs with $\mathcal{O}(|\mathcal{U}|)$ queries to the oracle returning the $\mu$-mixture distribution over items of the slate.
Theorem \ref{thm:CKT} below contains more detailed statements.
The idea relies on the fact that any uniform mixture of two MNLs over a $3$-item universe is identifiable,
and one can reconstruct the weight functions by querying to $2$- and $3$-slates, 
i.e. subsets $S$ with $|S| = 2, 3$.
But this algorithm fails for a non-uniform mixture of two MNLs, 
since a non-uniform mixture of two MNLs on a $3$-item universe is not necessarily identifiable (see Section \ref{sc3}, or \cite[Theorem 3]{CKT18}).
Nevertheless, the latter is `rarely' the case which is the first main result of this paper.

\begin{theorem}
\label{thm:main}
Let $n: = |\mathcal{U}| \ge 3$, and $\mu \in (0,1)$.
If the $\mu$-mixtures of MNLs $(a,b)$ and $(a',b')$ agree, i.e. have the same distribution on each $S \subseteq \mathcal{U}$, 
then 
\begin{equation*}
(a,b) \ne (a', b') \quad \mbox{implies} \quad R_n(a', b') = 0,
\end{equation*}
where $R_n$ is some polynomial specified later in the proof.
%Consequently, if $a = (a_1, \ldots, a_n)$ and $b = (b_1, \ldots, b_n)$ are drawn from two independent distributions with continuous density on the $(n-1)$-simplex $\Delta_{n-1}$, 
%then a $\mu$-mixture of MNLs $(a, b)$ is identifiable almost surely.
\end{theorem}

\quad Theorem \ref{thm:main} shows that the identifiability of a mixture of two MNLs may only fail on some {\em algebraic variety} $R_n = 0$.
We refer to the books of Sturmfels \cite{MS21, Sturm02} for a gentle introduction to algebraic geometry with focus on the computational aspects.
More important than the theorem itself is the way to construct the multivariate polynomials $R_n$.
As we will see later, this and the problem of learning a mixture of two MNLs essentially boils down to the problem of finding the common roots of some univariate quartic equations.
Unlike \cite{CKT18}, we adopt purely an equation-solving approach which is more natural and transparent.

\quad Now we turn to the polynomial identifiability of the mixture model.
There is a rich body of literature, e.g. \cite{AGH14, BCV14, GHK15, HK13, MV10} on the polynomial identifiability of Gaussian mixture models.
The algorithms use the method of moments, which relies on rich structure in lower-order moments of the Gaussian mixture model.
However, this low rank property is not enjoyed by the mixture of (two) MNLs, and lower moments are not sufficient to characterize or to approximate the model parameters.
In fact, it is even unknown whether all $\mu$-mixtures of two MNLs over a $4$-item universe are identifiable.
Based on numerical experiments, we conjecture the latter to be true. 
This motivates the following {\em $k$-identifiability} assumption for some $k \ge 4$, which will be used to devise a polynomial learning algorithm for the $\mu$-mixture model.

\begin{Assumption}[$k$-identifiability]
\label{assump:k}
The $\mu$-mixtures of two MNLs over a $k$-item universe are identifiable.
\end{Assumption}

\begin{theorem}
\label{thm:mainbis}
Let Assumption \ref{assump:k} hold. 
For $n \ge k$, consider a $\mu$-mixture of two MNLs over an $n$-item universe with the weight functions $(a,b)$.
Also assume that there are $0 < c \leq C < 1$ such that for $u \in S$,
\begin{equation}
\label{eq:eqcd}
\frac{c}{|S|} \le \mathbb{P}^{\mu}(u|S) \le \frac{C}{|S|}, \quad \mbox{as } |S| \rightarrow \infty.
\end{equation}
Let $\epsilon > 0$ be sufficiently small.
There is an adaptive algorithm which outputs weights $(\widehat{a}, \widehat{b})$ such that 
\begin{equation*}
\max_{1 \le i \le n} \left(\frac{|\widehat{a}_i - a_i|}{a_i} + \frac{|\widehat{b}_i - b_i|}{b_i} \right) \le \epsilon,
\end{equation*}
with high probability, by using
\begin{itemize}[itemsep = 3 pt]
\item
(sample complexity) $\mathcal{O}(n^3/\epsilon^2)$ independent samples of $n$- and $(n-1)$-slates.
\item
(query complexity) \, $3n + \mathcal{O}(1)$ queries to the estimated $n$- and $(n-1)$-slates.
\end{itemize}
\end{theorem}

\quad Theorem \ref{thm:mainbis} shows the polynomial identifiability of the $\mu$-mixture model under suitable conditions.
In addition to Assumption \ref{assump:k}, we postulate that the items in the slate are comparable to each other, which leads to \eqref{eq:eqcd}. 
Note that our algorithm takes half the number of queries that the adaptive algorithm in \cite{CKT18} uses to learn a uniform mixture of two MNLs, but it requires querying to large slates.

\quad The contributions of the paper are threefold:
\begin{enumerate}
\item[(i)]
We show that the identifiability of the mixture model does not cause much a problem, and it may only fail on some `small' algebraic variety of a negligible measure.
This is important to develop further Baysian nonparametric models, where user choice is modeled by a mixture of random MNLs.
That is, both $\mu$ and $(a,b)$ are random, and are drawn from some distributions. 
\item[(ii)]
We show that identifying or learning a mixture of two MNLs is reduced to solving a system of univariate quartic equations.
This gives a possible way to prove Assumption \ref{assump:k} -- the identifiability of any mixture of two MNLs over some finite universe.
\item[(iii)]
We provide a learning algorithm for an arbitrary mixture of two MNLs, which requires
a polynomial number of samples, and a linear number of queries. 
To the best of our knowledge, this paper is the first one that discusses the polynomial identifiability of mixtures of MNLs.
\end{enumerate}
As will be clear in the later proof, even the mixing parameter $\mu$ is unknown, Theorem \ref{thm:main} holds in the sense that the identifiability of $(\mu, a, b)$ may only fail on a negligible set.
However, it is not clear how to estimate efficiently the parameter $\mu$ from sample data.
The remaining issues, which seem to be technically challenging, are polynomial conditions in nature.
We hope that this work will draw attention to experts of algebraic geometry and symbolic computations,
so that advanced techniques in these domains can be used or developed to solve the conjectures in the paper.

\quad To conclude the introduction, let us mention a few relevant references.
There are a line of works discussing heuristic approaches to learn mixtures of MNLs by simulation \cite{Ge08, GB13, HJR03, Train09}.
Mixtures of MNLs have also been studied in the context of revenue maximization by \cite{BGG16, RSTT}.
More related to this work are \cite{CKT182, OS14, ZPX16, ZX19}, where different oracles are assumed. 
We refer to \cite[Section 2]{CKT18} for a more detailed explanation of the aforementioned references, 
and various pointers to other related works.

\quad The rest of the paper is organized as follows.
Section \ref{sc2} collects results related to mixtures of MNLs, and compares with our results.
Section \ref{sc3} warms up with a discussion of the mixture of two MNLs over a $3$-item universe.
Section \ref{sc4} studies the general problem of learning a mixture of two MNLs over an $n$-item universe.
Section \ref{sc5} gives the conclusion.

%-------------------------------------------------------------------------------------------------
\section{Preliminaries, existing results and comparison}
\label{sc2}

\quad This section provides background on the multinomial logit models, and recalls a few existing results in comparison with Theorems \ref{thm:main} and \ref{thm:mainbis}.
We follow closely the presentation in Chierichetti et al. \cite{CKT18}.
Consider the $n$-item universe, whose items are labeled by $[n]: = \{1, \ldots, n\}$. 
A {\em slate} is a non-empty subset of $[n]$, and a $k$-slate is a slate of size $k$.

\quad A {\em multinomial logit}, or simply {\em $1$-MNL} is determined by a weight function $a: [n] \to \Delta_{n-1}$,
where $\Delta_{n-1}:=\{(a_1, \ldots, a_n) \in \mathbb{R}^n_{+}: \sum_{i = 1}^n a_i = 1\}$ is the $(n-1)$-simplex, 
with $\mathbb{R}_{+}$ the set of nonnegative real numbers.
In this choice model, given a slate $T \subset [n]$, the probability that item $i \in T$ is selected is given by
\begin{equation}
\label{eq:1MNL}
D_T^a(i) = \frac{a_i}{\sum_{j \in T} a_j}, \quad i \in T.
\end{equation}
One can also take a weight function $a: [n] \to \mathbb{R}_{+}$, and normalizing each $a_i$ by $\sum_{j = 1}^n a_j$ will not affect the selection probability \eqref{eq:1MNL}.
A mixture of two multinomial logits, or simply {\em $2$-MNL} $\mathcal{A} = (a,b,\lambda)$ is specified by two weight functions $a, b$, and a mixing parameter $\lambda > 0$ in such a way that the probability that item $i \in T$ is selected in the slate $T$ is
\begin{equation}
\label{eq:2MNL}
D_T^{\mathcal{A}}(i) = \frac{1}{1+ \lambda} \frac{a_i}{\sum_{j \in T} a_j}  + \frac{\lambda}{1+\lambda} \frac{b_i}{\sum_{j \in T} b_j}, \quad i \in T.
\end{equation}
For later simplification, we use the parameter $\lambda > 0$ instead of $\mu:=1/(1+\lambda) \in (0,1)$ as the mixing weight.
So given a slate $T \subset [n]$, $\mathcal{A}$ first chooses the weight function $a$ with probability $1/(1+\lambda)$ and $b$ with probability $\lambda/(1+\lambda)$, and then behaves as the corresponding $1$-MNL.
For ease of presentation, we drop the superscript and write $D_T$ instead of $D^{\mathcal{A}}_T$ if there is no ambiguity.
If $\lambda = 1$ or $\mu = 1/2$, the choice model is called a uniform $2$-MNL.

\quad In \cite{CKT18}, the authors considered the problem of reconstructing the parameters $(a,b)$ of the uniform $2$-MNL, assuming an oracle access to $D_T(i)$ for all $T \subset [n]$ and $i \in T$.
Their results are summarized in the following theorem.

\begin{theorem}
\label{thm:CKT}
Let $n \ge 3$, and $\mathcal{A} = (a,b,1)$, $\mathcal{A}' = (a', b', 1)$ be two uniform $2$-MNL over $[n]$. 
Then:
\begin{itemize}[itemsep = 3 pt]
\item[(i)]
$\mathcal{A}$ and $\mathcal{A}'$ agree on each $T \subset [n]$, i.e. $D_T^{\mathcal{A}} = D_T^{\mathcal{A}'}$ for each $T \subset [n]$ if and only if 
$a = a'$, $b = b'$, or $a = b'$, $b = a'$.
\item[(ii)]
Any adaptive algorithm for $2$-MNL which queries to $k$-slates requires $\Omega(n/k)$ queries.
\item[(iii)]
There is an adaptive algorithm to learn a uniform $2$-MNL with $6n + \mathcal{O}(1)$ queries to $2$- and $3$-slates.
\end{itemize}
\end{theorem}

\quad As explained in the introduction, a weakness of \cite{CKT18} is that they assume the distribution of the slate $D_T$ is known. 
So they learn from the oracle distribution, not from the user samples.
In order to adapt the results of \cite{CKT18} for samples, we need the following elementary result.

\begin{lemma}
\label{lem:sample}
Let $p = (p_1, \ldots, p_n)$ be a distribution over $[n]$, 
and $\widehat{p}$ be the empirical distribution from $N$ independent samples of $p$:
\begin{equation*}
\widehat{p}_i = \frac{1}{N}\sum_{\ell = 1}^N 1(X_{\ell} = i), \quad 1 \le i \le n,
\end{equation*}
where $X_1, \ldots, X_N$ are i.i.d. copies from the distribution $p$.
Assume that there are $0 < c \le C < 1$ such that
\begin{equation}
\label{eq:bound}
\frac{c}{n} \le p_i \le \frac{C}{n}.
\end{equation}
If $N \gg n^3$, then with probability $1 - \mathcal{O}(n^3/N)$,
\begin{equation*}
\sup_{1 \le i \le n} |p_i - \widehat{p}_i| \le c_0 \left(\frac{n}{N}\right)^{1/2} \quad \mbox{for some } c_0 > 0.
\end{equation*}
\end{lemma}
\begin{proof}
Note that $\widehat{p}_i$ is Binomial$(p_i, N)$.
Let $\sigma_i = \sqrt{p_i (1-p_i)}$, and $\rho_i: = \mathbb{E}|\widehat{p}_i - p_i|^3$.
By assumption \eqref{eq:bound}, we have
\begin{equation*}
\sigma_i \asymp \frac{1}{\sqrt{n}}, \quad \rho_i \asymp \frac{1}{n}.
\end{equation*}
Let $Y: = \sqrt{N}(p_i - \widehat{p}_i)/\sigma$ and $Z \sim $ Normal$(0,1)$. 
By the Berry-Esseen bound, for any $t > 0$,
\begin{equation}
\label{eq:BE}
|\mathbb{P}(|Y_i| >t) - \mathbb{P}(|Z| > t)| \le \frac{2 \rho_i}{\sigma^3 \sqrt{N}}
\end{equation}
By taking $t = n$ in \eqref{eq:BE}, we have $|p_i - \widehat{p}_i| > \sqrt{n/N}$ with probability at most
$\mathbb{P}(Z > n) + \sqrt{n/N}$.
The union bound yields
\begin{equation*}
\mathbb{P}(\cup_{i = 1}^n \{|p_i - \widehat{p}_i| > \sqrt{n/N}\}) \le \sqrt{\frac{n^3}{N}} + n \mathbb{P}(|Z| > n).
\end{equation*}
This leads to the desired result when $N \gg n^3$.
\end{proof}

\quad As an easy consequence of Lemma \ref{lem:sample} and \cite[Theorems 5]{CKT18}, 
the following result shows the polynomial identifiability of the uniform mixture of two MNLs.
\begin{theorem}
\label{thm:sample}
Consider a uniform mixture of two MNLs over an $n$-item universe with the weight functions $(a,b)$.
Assume that \eqref{eq:eqcd} holds for $\mu = 1/2$.
Let $\epsilon > 0$ be sufficiently small.
There is an adaptive algorithm which outputs weights $(\widehat{a}, \widehat{b})$ such that 
\begin{equation*}
\max_{1 \le i \le n} \left(\frac{|\widehat{a}_i - a_i|}{a_i} + \frac{|\widehat{b}_i - b_i|}{b_i} \right) \le \epsilon,
\end{equation*}
with high probability, by using
\begin{itemize}[itemsep = 3 pt]
\item
(sample complexity) $\mathcal{O}(1/\epsilon^2)$ independent samples of $2$- and $3$-slates.
\item
(query complexity) \, $6n + \mathcal{O}(1)$ queries to the estimated $2$- and $3$-slates.
\end{itemize}
\end{theorem}

\quad Now we compare Theorem \ref{thm:sample} with Theorem \ref{thm:mainbis}. 
For a uniform mixture of two MNLs, it requires $\mathcal{O}(1/\epsilon^2)$ samples of $2$- and $3$-slates to achieve a given accuracy $\epsilon$.
This has advantage in both sample size and slate size over the $\mathcal{O}(n^3/\epsilon^2)$ samples of $\mathcal{O}(n)$-slates in our algorithm.
However, the algorithm in \cite{CKT18} for the uniform mixture relies on the identifiability of the model over a $3$-item universe,
which is not available for an arbitrary mixture of two MNLs.
Given the estimated slates, our algorithm for an arbitrary mixture requires half the number of queries that is used to learn a uniform mixture model.
It is interesting to know whether there is an algorithm which only queries to small slates with polynomial sample complexity, and linear query complexity.
We leave the problem open.

%-------------------------------------------------------------------------------------------------
\section{Identifiability of $2$-MNL on the $3$-item universe}
\label{sc3}

\quad In this section, we study a $2$-MNL on the $3$-item universe $\{1,2,3\}$.
Assume that the mixing parameter $\lambda > 0$ is known.
So we only need to reconstruct the $1$-MNL weights $a = (a_1, a_2, a_3)$ and $b = (b_1, b_2, b_3)$.
As pointed out in \cite{CKT18}, for $\lambda \ne 1$ the oracle $(D_{\{1,2,3\}}(\cdot), D_{\{1,2\}}(\cdot), D_{\{1,3\}}(\cdot),$\\
$D_{\{2,3\}}(\cdot))$ does not uniquely determine the weights $(a_1, a_2, a_3, b_1, b_2, b_3)$.
The main point of Theorem \ref{thm:main} is that this situation rarely happens, and we can characterize the instances 
where the uniqueness fails.
As will be seen in Section \ref{sc4}, the computation yielding the non-uniqueness characterization for $n = 3$ is 
a building block to study the identifiability of $2$-MNL for $n > 3$.

\quad Let the oracle $\left(D_{\{1,2,3\}}, D_{\{1,2\}}, D_{\{1,3\}}, D_{\{2,3\}}\right)$ be generated from the weights $(a', b') \in \Delta_2 \times \Delta_2$, e.g. $D_{\{1,2\}}(1) = \frac{1}{1 + \lambda}\frac{a'_1}{a'_1 + a'_2} + \frac{\lambda}{1 + \lambda}\frac{b'_1}{b'_1 + b'_2}$.
To simplify the notations, we denote $C_T: = (1 + \lambda) D_T$.
The problem involves solving the following system of equations:
\begin{subequations}
\label{eq:31}
\begin{align}
\displaystyle &\frac{a_i}{a_i + a_j} + \lambda \frac{b_i}{b_i + b_j} = C_{\{i,j\}}(i)  \qquad \mbox{for } i,j \in [3] \mbox{ and } i \ne j,  \label{eq:31a}\\
\displaystyle &a_i + \lambda b_i = C_{\{1,2,3\}}(i)  \qquad \qquad \qquad \mbox{for } i \in [3], \label{eq:31b} \\ 
\displaystyle &a_1 + a_2 + a_3 = b_1 + b_2 + b_3 = 1. \label{eq:31c}
\end{align}
\end{subequations}

So there are $6$ unknowns and 11 equations, $6$ from \eqref{eq:31a}, $3$ from \eqref{eq:31b} and $2$ from \eqref{eq:31c}.
Since
$\left(\frac{a_i}{a_i + a_j} + \lambda \frac{b_i}{b_i + b_j} \right) + \left(\frac{a_j}{a_i + a_j} + \lambda \frac{b_j}{b_i + b_j} \right) = 1 + \lambda$ and
$\sum_{i = 1}^3 (a_i + \lambda b_i) = 1 + \lambda$,
there are $7$ linearly independent equations. 
Implicitly, there are $6$ more inequalities: $0 \le a_i \le 1$, $0 \le b_i \le 1$ for $i \in [3]$.
The following lemma provides a simple way to narrow down the possible solutions to \eqref{eq:31}.
\begin{lemma}
\label{lem:algebra}
Assume that $(a_1, a_2, a_3, b_1, b_2, b_3)$ solves \eqref{eq:31}, and $b_1 \ne \frac{C_{\{1,2,3\}}(1)}{1 + \lambda}$.
Then:
\begin{itemize}[itemsep = 3 pt]
\item[(i)]
Each of $a_1, a_2, a_3, b_2, b_3$ can be written as a simple function of $b_1$, which is specified by \eqref{eq:32}--\eqref{eq:33}.
\item[(ii)]
$b_1$ solves a quartic equation given by \eqref{eq:34}.
\end{itemize}
\end{lemma}

\begin{proof}
By \eqref{eq:31b}--\eqref{eq:31c}, it is easy to see that $(b_1, b_2)$ determines the remaining variables by
\begin{equation}
\label{eq:32}
\begin{aligned}
& a_1 = C_{\{1,2,3\}}(1) - \lambda b_1, \quad a_2 = C_{\{1,2,3\}}(2) - \lambda b_2, \quad b_3 = 1 - b_1 - b_2, \\
& a_3 = 1 - C_{\{1,2,3\}}(1) - C_{\{1,2,3\}}(2) + \lambda(b_1 + b_2).
\end{aligned}
\end{equation}
Note that \eqref{eq:31a} can be rewritten as
\begin{equation}\tag{3.1a'}
\label{eq:31abis}
\frac{a_i}{1 - a_k} + \lambda \frac{b_i}{1 - b_k} = C_{\{i,j\}}(i) \qquad \mbox{for }  i \ne j, \, k \ne i, j.
\end{equation}
Specializing \eqref{eq:31abis} to $i = 2$, $j = 3$ and using \eqref{eq:32} to express $a_1,a_2$ in terms of $b_1, b_2$, we get
\begin{equation}
\label{eq:33}
b_2 = \frac{\left(C_{\{2,3\}}(2)(1 - C_{\{1,2,3\}}(1) + \lambda b_1) - C_{\{1,2,3\}}(2)\right)(1-b_1)}{\lambda \left((1+\lambda) b_1 - C_{\{1,2,3\}}(1) \right)}=: \frac{N(b_1)}{D(b_1)},
\end{equation}
where $N(b_1)$ is a quadratic function of $b_1$, and $D(b_1)$ is linear in $b_1$.
Further specializing \eqref{eq:31abis} to $i = 1$, $j = 3$ and using \eqref{eq:32}--\eqref{eq:33} to express $a_1, a_2, b_2$ in terms of $b_1$, we have
\begin{multline}
\label{eq:34}
C_{\{1,2,3\}}(1) \left((1 - C_{\{1,2,3\}}(2)) D(b_1) + \lambda N(b_1) \right) (D(b_1) - N(b_1)) \\
-(C_{\{1,2,3\}}(1) - \lambda b_1)(D(b_1) - N(b_1)) - \lambda b_1 \left((1 - C_{\{1,2,3\}}(2)) D(b_1) + \lambda N(b_1) \right) =0.
\end{multline}
The first term on the l.h.s. of \eqref{eq:34} is a quartic polynomial in $b_1$, and the other two terms are cubic polynomials in $b_1$.
\end{proof}

\quad The main point of Lemma \ref{lem:algebra} is to reduce the system of equations \eqref{eq:31} to that only involving the $4$-tuple $(a_1, a_2, b_1,b_2)$:
\begin{subequations}
\label{eq:35}
\begin{align}
\displaystyle &\frac{a_1}{1-a_2} + \lambda \frac{b_1}{1-b_2} = C_{\{1,3\}}(1), \quad \frac{a_2}{1-a_1} + \lambda \frac{b_2}{1-b_1} = C_{\{2,3\}}(2),  \label{eq:35a}\\
\displaystyle &a_1 + \lambda b_1= C_{\{1,2,3\}}(1), \quad  a_2 + \lambda b_2= C_{\{1,2,3\}}(2). \label{eq:35b}
\end{align}
\end{subequations}
Moreover, solving the system of equations \eqref{eq:35} is equivalent to solving a univariate quartic equation.
We call the equations \eqref{eq:35a}--\eqref{eq:35b} the {\em $(a_1,a_2,b_1,b_2)$-system}.
Note that if $b_1 = \frac{C_{\{1,2,3\}}(1)}{1 + \lambda}$ and $b_2 \ne  \frac{C_{\{1,2,3\}}(2)}{1 + \lambda}$, then similar to \eqref{eq:33} we can express $b_1$ in terms of $b_2$, and hence all the other variables in terms of $b_2$ by \eqref{eq:32}.
If $b_1 = \frac{C_{\{1,2,3\}}(1)}{1 + \lambda}$ and $b_2 =  \frac{C_{\{1,2,3\}}(2)}{1 + \lambda}$, it is clear that all the other variables are uniquely determined by \eqref{eq:32}.

\quad Recall that the values of  $\left(C_{\{1,2,3\}}, C_{\{1,2\}}, C_{\{1,3\}}, C_{\{2,3\}}\right)$ are generated from some weights $(a', b')$. 
This implies that the quartic equation $P_{12}(b_1) = 0$ given by \eqref{eq:34} has at least one real root in $[0,1]$, which is $b'_1$.
Here the subscript `$12$' indicates that the polynomial $P_{12}$ is associated with the $(a_1,a_2,b_1,b_2)$-system.
Let 
\begin{equation*}
Q_{12}(b_1): = \frac{P_{12}(b_1)}{(b_1 - b'_1)},
\end{equation*}
which is a cubic polynomial whose coefficients are rational functions of $(a'_1, a'_2, b'_1, b'_2)$.
Therefore, the identifiability of a $2$-MNL on the $3$-item universe, or equivalently the uniqueness of the solution to the system of equations \eqref{eq:31} reduces to the problem $(i)$ if the cubic polynomial $Q$ has a real roots $b''_1 \in [0,1]$ and $b''_1 \ne b'_1$; 
$(ii)$ if the corresponding $6$-tuple $(a''_1, a''_2, a''_3, b''_1, b''_2, b''_3)$ given by \eqref{eq:32}--\eqref{eq:33}
solves \eqref{eq:31}.
Algorithmically, this is rather easy to verify: Cardano's formula \cite{Cardano} solves any cubic equation. 
Then it suffices to check 
if the corresponding $(a''_1, a''_2, a''_3, b''_1, b''_2, b''_3) \in [0,1]^6$, and 
if the equation $\frac{a''_2}{a''_2 + a''_3} + \lambda \frac{b''_2}{b''_2+ b''_3} = C_{\{2,3\}}(2)$, which is part of \eqref{eq:31} but not in \eqref{eq:35} holds.

\quad Now we show that it is rarely the case that the system of equations \eqref{eq:31} has more than one solution $(a,b) \in \Delta_2 \times \Delta_2$.
In fact, it is even true that \eqref{eq:31} can barely have more than one solution $(a,b) \in \mathbb{R}^3 \times \mathbb{R}^3$.
\begin{proof}[Proof of Theorem \ref{thm:main} (n = 3)]
Consider $P_{12}(b_1), Q_{12}(b_1)$ associated with the $(a_1,a_2,b_1,b_2)$-system, and $P_{13}(b_1), Q_{13}(b_1)$ associated with the $(a_1,a_3,b_1,b_3)$-system.
Observe that the system \eqref{eq:31} has only one solution if the polynomials $P_{12}$ and $P_{13}$ have only one common root $b_1 = b'_1$, or equivalently the polynomials $Q_{12}$ and $Q_{13}$ do not have any common root. 
It is well known \cite[Lemma 3.6]{Harris95} that the latter holds if and only if 
\begin{equation*}
\mbox{Res}(Q_{12}, Q_{13}) \ne  0,
\end{equation*}
where $\mbox{Res}(Q_{12}, Q_{13})$ is the {\em resultant} of $Q_{12}$ and $Q_{13}$, the determinant of a $6 \times 6$ Sylvester matrix whose entries are the coefficients of $Q_{12}$ and $Q_{13}$.
Recall that the coefficients of $Q_{12}$, $Q_{13}$ are rational functions of $(a', b') = (a'_1, a'_2, a'_3, b'_1, b'_2, b'_3)$. 
By letting $R_3(a',b')$ be the polynomial corresponding to the numerator of $\mbox{Res}(Q_{12}, Q_{13})$, we have
\begin{multline*}
\left\{(a', b') \in \Delta_2 \times \Delta_2: \eqref{eq:31} \mbox{ has more than one solution} \right\} \\
\subset \left\{(a', b') \in \mathbb{R}^3 \times \mathbb{R}^3: R_3(a',b') = 0  \right\}.
\end{multline*}
That is, the uniqueness of the solution to \eqref{eq:31} may fail only for those $(a', b')$ on the algebraic variety $R_3(a', b') = 0$.
\end{proof}

\quad Of course, it is rather impossible to put down the expression of $R_3$ by hand. 
With the help of \texttt{Mathematica}, we get an expression of $R_3(a, b)/\lambda^6$ as follows:
\begin{figure}[h]
\includegraphics[width=1\columnwidth]{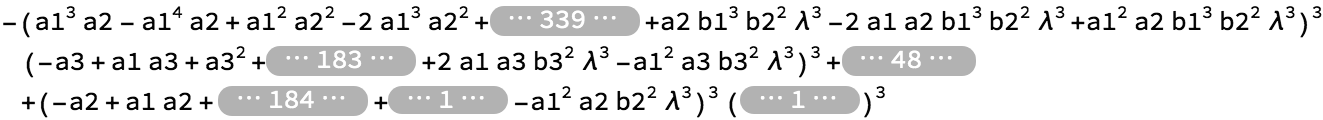}
\end{figure}

But it seems that even \texttt{Mathematica} finds it challenging to expand, or to simplify the above large expression.
%We only know that the maximum degrees of $a_1, a_2, a_3, b_1,b_2, b_3$ appearing in $R_3$ are $26,9,9,23,9,9$ respectively.

\quad To conclude this section, we go back to the cubic polynomial $Q_{12}$.
This polynomial has either $3$ real roots, or $1$ real root and $2$ complex conjugate roots.
Since we are only concerned with the real roots of $Q_{12}$, it is natural to ask whether $Q_{12}$ can have $3$ real roots for some $\lambda > 0$ and $(a',b') \in \Delta_2 \times \Delta_2$.
The point is that if $Q_{12}$ has only $1$ real root, the analysis may further be simplified.
It turns out that this question is subtle. 
It is known \cite{Cardano} that $Q_{12}$ has $3$ real roots if the {\em discriminant} of $Q_{12}$ is nonnegative.
Note that the discriminant of $Q_{12}$ is a function of $(a'_1, a'_2, b'_1, b'_2)$.
For $\lambda = 2$, the functions \texttt{FindInstance} and \texttt{NSolve} in \texttt{Mathematica} do not find any $4$-tuple $(a'_1,a'_2,b'_1, b'_2) \in \Delta_1 \times \Delta_1$ such that the discriminant of $Q_{12}$ is nonnegative.
Moreover, the function \texttt{NMaximize} finds numerically the maximum of the discriminant of $Q_{12}$ over $\Delta_1 \times \Delta_1$, which is $-0.00317 < 0$.
These observations suggest that $Q_{12}$ have only $1$ real root when $\lambda = 2$.
However, for $\lambda = 5$ and $(a'_1,a'_2,b'_1, b'_2) = (0.0389099, 0.000870832, 0.0565171, 0.943483)$, 
\texttt{Mathematica} finds that $Q_{12}$ has $3$ real roots: $0.043916, 0.164599, 0.281671$.
Based on many experiments, we conjecture that there is a threshold $\lambda_{\tiny \mbox{thres}} > 0$ such that for $\lambda < \lambda_{\tiny \mbox{thres}}$, $Q_{12}$ has only $1$ real roots whatever the values of $(a'_1,a'_2,b'_1, b'_2)$, and for $\lambda \ge \lambda_{\tiny \mbox{thres}}$, $Q_{12}$ can have either $1$ or $3$ real roots depending on the values of $(a'_1,a'_2,b'_1, b'_2)$.

%-------------------------------------------------------------------------------------------------
\section{Learning $2$-MNL on the $n$-item universe}
\label{sc4}

\quad This section is devoted to the study of a $2$-MNL over $[n]$ for $n > 3$.
The key idea is to reduce the system of equations over $2n$ unknowns to $(a_i,a_j,b_i,b_j)$-systems.
This also gives a promising way to prove the identifiability of a $2$-MNL on $[n]$ for some possible $n > 3$.
The following result records the general structure of learning a $2$-MNL over $n$ items.
Recall the definitions of $D_T$ and $C_T$ for $T \subset [n]$ from Section \ref{sc3}.

\begin{proposition}
Let $n \ge 3$, and $\mathcal{A} = (a,b, \lambda)$ be a $2$-MNL over $[n]$.
Assume that $\lambda > 0$ is known. 
Then learning $\mathcal{A}$ is equivalent to solving the following system of equations with $2n$ unknowns $(a_1, \ldots, a_n, b_1, \ldots, b_n) \in \mathbb{R}_{+}$:
\begin{subequations}
\label{eq:41}
\begin{align}
\displaystyle &\frac{a_i}{\sum_{j \in T} a_i} + \lambda \frac{b_i}{\sum_{j \in T} b_j} = C_T(i)  \qquad \mbox{for } T \subset [n] \mbox{ with } |T| \ge 2, \mbox{ and } i \in T,  \label{eq:41a}\\
\displaystyle & \sum_{i = 1}^n a_i = \sum_{i = 1}^n b_i = 1. \label{eq:41b}
\end{align}
\end{subequations}
So there are $2 + 2^{n-1} n - n$ equations, and at most $3 + 2^{n-1} (n -2)$ of these equations are linearly independent.
\end{proposition}
\begin{proof}
Note that for each $T \subset [n]$ with $|T| = k$, \eqref{eq:41a} contributes $k$ equations, and $k-1$ of these equations are linearly independent.
The result follows from the well known identities $\sum_{k = 1}^n k \binom{n}{k} =2^{n-1} n $ and $\sum_{k = 1}^n (k-1) \binom{n}{k} =  1 + 2^{n-1}(n-2)$.
\end{proof}

\quad The following $(a_i, a_j, b_i, b_j)$-system is an obvious extension of \eqref{eq:35} to the $n$-item universe:
\begin{subequations}
\label{eq:42}
\begin{align}
\displaystyle &\frac{a_i}{1-a_j} + \lambda \frac{b_i}{1-b_j} = C_{[n] \setminus \{j\}}(i), \quad \frac{a_j}{1-a_i} + \lambda \frac{b_j}{1-b_i} = C_{[n] \setminus \{i\}}(j),  \label{eq:42a}\\
\displaystyle &a_i + \lambda b_i= C_{[n]}(i), \quad  a_j + \lambda b_j= C_{[n]}(j). \label{eq:42b}
\end{align}
\end{subequations}
Similar to Lemma \ref{lem:algebra}, solving the system of equations \eqref{eq:41} boils down to solving a univariate quartic equation by using $(a_i,a_j,b_i,b_j)$-systems.

\begin{lemma}
\label{lem:algebra2}
Assume that $(a_1, \ldots, a_n, b_1, \ldots, b_n)$ solves \eqref{eq:41}, and $b_1 \ne \frac{C_{\{1,2,3\}}(1)}{1 + \lambda}$.
Then:
\begin{itemize}[itemsep = 3 pt]
\item[(i)]
If for each $i$ there exists a set $T(i)$ containing $i$ such that $b_i = \frac{C_T(i)}{1 +\lambda}$, 
then $a_i = C_{[n]}(i) - \frac{\lambda}{1 + \lambda} C_{T(i)}(i)$.
\item[(ii)]
Otherwise, assume without loss of generality $b_1 \ne \frac{C_T(1)}{1 + \lambda}$ for any set $T$ containing $i$. 
Then each of $a_1, \ldots, a_n, b_2, \ldots b_n$ can be written as a simple function of $b_1$,
and $b_1$ solves a quartic equation.
\end{itemize}
\end{lemma}
\begin{proof}
Part $(i)$ is straightforward. 
For part $(ii)$, consider the $(a_1,a_i, b_1, b_i)$-system for all $i \in \{2, \ldots, n\}$.
By Lemma \ref{lem:algebra}, $a_1, a_i, b_i$ are fully determined by $b_1$, 
and $b_1$ solves the quartic equation $P_{12} = 0$ associated with the  $(a_1, a_2, b_1, b_2)$-system.
\end{proof}

\quad Basically, Lemma \ref{lem:algebra2} shows that the system \eqref{eq:41} with approximately $2^{n-1} n$ equations
has at most $4$ solutions.
This simplification only makes use of $2n$ equations, which involves those in \eqref{eq:41a} with $|T| = n-1, n$. 
That is, queries to $(n-1)$-slates and $n$-slates.
Theorem \ref{thm:main} is then a consequence of Lemma \ref{lem:algebra2}.

\begin{proof}[Proof of Therorem \ref{thm:main} (general $n$)]
We follow the notations in Section \ref{sc3} which proves the result for $n = 3$.
Similarly, define the polynomials $P_{ij}$, $Q_{ij}$ associated with the $(a_i, a_j, b_i, b_j)$-system.
Recall that the coefficients of $P_{ij}$, $Q_{ij}$ are rational functions of $(a'_i, a'_j, b'_i, b'_j)$.
By Lemma \ref{lem:algebra2}, if the system of equations \eqref{eq:41} has more than one solution, 
then the polynomials $Q_{12}, Q_{13}, \ldots, Q_{1n}$ have a common root.
The latter implies that each pair $(Q_{1j}, Q_{1k})$ have a common root, which is equivalent to
\begin{equation*}
\mbox{Res}(Q_{1j}, Q_{1k}) = 0 \quad \mbox{for } j, k \in \{2, \ldots, n\} \mbox{ and } j < k,
\end{equation*}
where $\mbox{Res}(\cdot, \cdot)$ is the resultant of two polynomials.
Let $W_{1jk}$ be the multivariate polynomial in $(a'_1, a'_j, a'_k, b'_1, b'_j, b'_k)$ corresponding to the numerator of $\mbox{Res}(Q_{1j}, Q_{1k})$.
We have 
\begin{multline*}
\left\{(a', b') \in\Delta_{n-1} \times \Delta_{n-1}: \eqref{eq:41} \mbox{ has more than one solution} \right\} \\
 \subset \left\{(a', b') \in \mathbb{R}^n \times \mathbb{R}^n: W_{1jk} = 0 \mbox{ for all } 2 \le j < k \le n \right\}.
\end{multline*}
In particular, one can take $R_4 (a',b') = \sum_{2 \le j < k \le n} (W_{1jk})^2$ so that \eqref{eq:41} has the unique solution 
when $R_4(a',b') \ne 0$.
\end{proof}

\quad There are many ways to build $R_n$ in Theorem \ref{thm:main}. 
One can take $R_n = \sum_{(j,k) \in \mathcal{S}} (W_{1jk})^2$ for $\mathcal{S}$ any subset of $\{(j,k): 2 \le j < k \le n\}$, 
e.g. $R_n = \sum_{2 \le j < k \le n} (W_{jk})^2$ in the previous proof, or just $R_n = W_{1jk}$ for some $j, k$.
Given $(a,b) \in \Delta_{n-1} \times \Delta_{n-1}$, it is relatively easy to check numerically if $R_{n}(a,b) = 0$.
But as mentioned in Section \ref{sc3}, even in the simple case $R_n = W_{1jk}$ (with $6$ variables), \texttt{Mathematica} cannot output the expression of $R_n$, let alone doing further symbolic computations.

\quad Now we give a simpler $R_n$ when $n > 3$.
Consider the equations in \eqref{eq:41a} relating only $(a_1, a_2, b_1, b_2)$.
For $n >3$, in addition to the $(a_1, a_2, b_1, b_2)$-system there is one more:
\begin{equation}
\label{eq:43}
\frac{a_1}{a_1 +a_2} + \lambda \frac{b_1}{b_1 + b_2} = C_{\{1,2\}}(1).
\end{equation}
By injecting \eqref{eq:32}--\eqref{eq:33} into \eqref{eq:43}, 
we get $\widetilde{P}_{12}(b_1) = 0$ with $\widetilde{P}_{12}$ another quartic polynomial.
Let $\widetilde{Q}_{12}(b_1) = \widetilde{P}_{12}(b_1)/(b_1 - b'_1)$ be the corresponding cubic polynomial.
Another simple choice for $R_n$ is the numerator of $\mbox{Res}(Q_{12}, \widetilde{Q}_{12})$, 
which \texttt{Mathematica} outputs:
\begin{figure}[h]
\includegraphics[width=1\columnwidth]{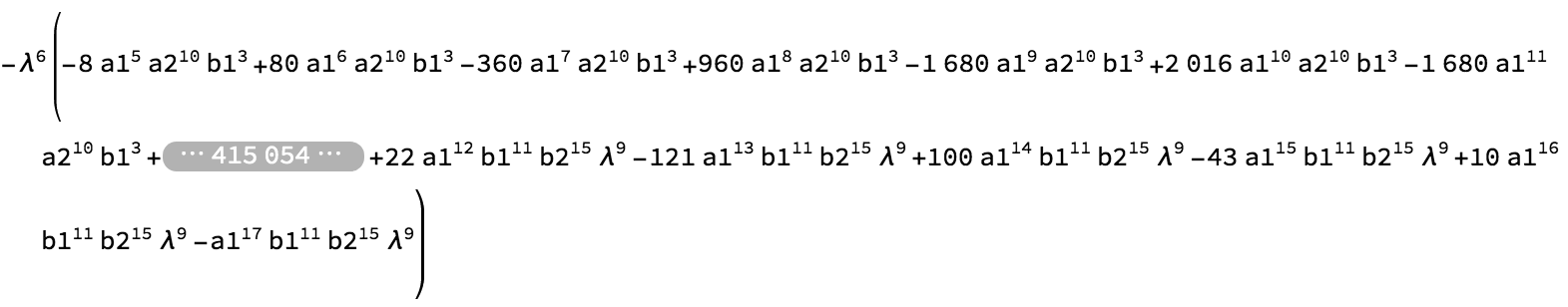}
\end{figure}

Such defined $R_n$ is a polynomial in $4$ variables, and is apparently simpler than the previous choices for $R_n$.
The expression of this $R_n$ has more than $400,000$ terms, and the maximum degrees of $a_1,a_2,b_1,b_2$ appearing in $R_n$ are $25,15,25,15$ respectively.
An interesting question is whether $\{R_n = 0\} \cap (\Delta_1 \times \Delta_1)$ is empty for all $\lambda > 0$.
If so, the system of equations \eqref{eq:42}--\eqref{eq:43} has a unique solution $(a'_1, a'_2,b'_1,b'_2)$ 
which implies that any $2$-MNL over $[n]$ is identifiable for $n  > 3$.
Unfortunately, the following example shows it is not the case.
For $\lambda = 2$, and $a'_1 = a'_2 = 2/5$, $b'_1 = b'_2 = 3/10$,
the system of equations \eqref{eq:42}--\eqref{eq:43} gives two solutions: $a_1 = a_2 = 2/5$, $b_1 = b_2 = 3/10$,
and $a_1 = a_2 = 5/19$, $b_1 = b_2 = 7/19$.

\quad In principle, determining if a $2$-MNL over $[n]$ is identifiable, or if the system of equations \eqref{eq:41} has only one solution reduces to determining if a set of about $2^{n-1} n$ univariate cubic equations have a common root.
The latter is equivalent to whether the {\em resolvent} of these cubic polynomials is zero, 
and this amounts to a fairly large number of multivariate polynomial equations on $(a',b')$, see e.g. \cite[Section 7]{Goldman06}.
So as $n$ increases, the set of $(a',b')$ for which the system of equations \eqref{eq:41} has more than $1$ solution
becomes more and more restrictive.
It is believable that there is a threshold $n_{\tiny \mbox{thre}} > 3$ such that \eqref{eq:41} has only one solution for $n \ge n_{\tiny \mbox{thre}}$.
Based on many experiments, we conjecture that $n_{\tiny \mbox{thre}} = 4$; that is, any $2$-MNL over $[n]$ for $n \ge 4$ is identifiable.
We leave this puzzle to interested readers.

\quad To finish this section, we prove Theorem \ref{thm:mainbis} on the polynomial identifiability of the mixture model.
\begin{proof}[Proof of Theorem \ref{thm:mainbis}]
We start with the identifiability, i.e. learn the weights from the oracle $C_T$.
Note that $(a_1, \ldots, a_k, b_1, \ldots, b_k)$ solves \eqref{eq:41a} for $n= k$ but with possibly $\sum_{j = 1}^k a_j$, $\sum_{j = 1}^k b_j \ne 1$.
By the $k$-identifiability, $(a_1, \ldots, a_k, b_1, \ldots, b_k)$ takes the form:
$a_j = a^{[k]}_j\Sigma_a$, $b_j = b^{[k]}_j \Sigma_b$ for $j \in [k]$,
where $(a^{[k]}_1, \ldots, a^{[k]}_k, b^{[k]}_1, \ldots, b^{[k]}_k)$ is the unique solution to \eqref{eq:41} for $n = k$, and $\Sigma_a = \sum_{j = 1}^k a_j$ and $\Sigma_b = \sum_{j = 1}^k b_j$ have yet to be determined.
It requires $\mathcal{O}(1)$ queries to find $(a^{[k]}_1, \ldots, a^{[k]}_k, b^{[k]}_1, \ldots, b^{[k]}_k)$.
For instance, by Lemma \ref{lem:algebra2}, there are at most $4$ possible solutions associated with the $(a_1,a_2,b_1,b_2)$-system, and a quartic equation is easily solved by Ferrari's method \cite{Cardano}. 
Then it suffices to check which one of these solutions satisfy other $2^{k-1}(k-2) - 3$ equations.
Next for each $j \in \{k+1, \ldots, n\}$, consider the $(a_1, a_j, b_1, b_j)$-system \eqref{eq:42} which has $3$ queries.
By Lemma \ref{lem:algebra2}, $b_j$ can be expressed in terms of $b_1$.
So the condition $\sum_{j = 1}^n b_j = 1$ determines $\Sigma_b$. 
Similarly, by expressing $a_j$ in terms of $a_1$, the condition $\sum_{j = 1}^n a_j = 1$ specifies $\Sigma_a$.
In total, it requires $\mathcal{O}(1) + 3(n -k) = 3n + \mathcal{O}(1)$ queries to learn the mixture model.

\quad Now we consider learning the weights from samples $\widehat{C}_T$. 
By Lemma \ref{lem:sample}, if $N \gg n^3$, we have $\sup_{1 \le i \le n}|\widehat{C}_T(i) - C_T(i)| \le c_0 (n/N)^{1/2}$ for $|T| = n-1, n$. 
Further by assumption \eqref{eq:eqcd}, we get
\begin{equation}
\label{eq:last}
\sup_{1 \le i \le n}\frac{|\widehat{C}_T(i) - C_T(i)|}{C_T(i)} \le c_1 \frac{n^{3/2}}{N^{1/2}},
\end{equation}
for some $c_1 > 0$.
The idea is to control the ratio $|\widehat{C}_T(i) - C_T(i)|/C_T(i)$ uniformly in $i$.
It is easily seen that for $N = \mathcal{O}(n^3/\epsilon^2)$, the term on the left side of \eqref{eq:last} is controlled by $\epsilon$, and we write $\widehat{C}_T(i) = C_T(i) (1 + \epsilon)$.
Since $\widehat{C}_T$ is not the oracle, the equation \eqref{eq:41} with $n = k$ and $\widehat{C}_T$ (instead of $C_T$) does not necessarily have a unique admissible solution. 
By continuity of complex roots with respect to polynomial coefficients, 
for $\epsilon$ sufficiently small, there is a unique root with real part between $(0,1)$ and the smallest imaginary part -- this gives $\widehat{b}^{[k]}_1$ and then $(\widehat{a}^{[k]}_1, \ldots, \widehat{a}^{[k]}_k, \widehat{b}^{[k]}_1, \ldots, \widehat{b}^{[k]}_k)$.
We repeat the procedure in the previous paragraph by substituting $C_T$ with $\widehat{C}_T$.
By the implicit function theorem, we have 
$\widehat{a}_j^{[k]} = a_j^{[k]}(1+ \mathcal{O}(\epsilon))$, $\widehat{b}_j^{[k]} = b_j^{[k]}(1+ \mathcal{O}(\epsilon))$, $\widehat{\Sigma}_a = \Sigma_a(1+\mathcal{O}(\epsilon))$ and $\widehat{\Sigma}_b = \Sigma_b(1+\mathcal{O}(\epsilon))$.
As a result,
\begin{equation*}
\frac{|\widehat{a}_j - a_j|}{a_j} = \frac{|\widehat{a}_j^{[k]} \widehat{\Sigma}_a - a^{[k]}_j\Sigma_a |}{a^{[k]}_j\Sigma_a} = \mathcal{O}(\epsilon).
\end{equation*}
Similar results hold for the weight function $b$.
\end{proof}

%-------------------------------------------------------------------------------------------------
\section{Conclusion}
\label{sc5}

\quad In this paper, we study the problem of learning an arbitrary mixture of two multinomial logits. 
We have proved that the identifiability of the mixture models is less problematic, since it may only fail on a set of parameters of a negligible meausure.
The proof is based on a reduction of the learning problem to a system of univariate quartic equations.
We also proposed an algorithm to learn any mixture of two MNLs using a polynomial number of samples and a linear number of queries, under the assumption that a mixture of two MNLs over some finite universe is identifiable. 

\quad The paper also leaves a few problem for future research.
For instance, 
\begin{enumerate}
\item[(i)]
prove all the conjectures in Sections \ref{sc3} and \ref{sc4}.
\item[(ii)]
devise an algorithm to learn an arbitrary mixture of two MNLs using only small slates with polynomial samples and queries.
\item[(iii)]
consider the identifiability issue for both the mixing parameter $\lambda$ and the weight functions $(a,b)$.
\item[(iv)]
study the problem of learning a mixture of more than two multinomial logits.
\end{enumerate}
We hope that our work will trigger further developments on the mixture models.

\bibliographystyle{abbrv}
\bibliography{unique}
\end{document}